\documentclass[11pt]{article}

\PassOptionsToPackage{numbers,sort&compress}{natbib}

\usepackage[abbrvbib, preprint]{jmlr2e}
\makeatletter
\renewcommand{\NAT@open}{[}
\renewcommand{\NAT@close}{]}
\makeatother

\makeatletter
\NAT@numberstrue
\makeatother

\let\cite\citep

\usepackage{blindtext}

\usepackage{lastpage}

\firstpageno{1}

\usepackage{graphicx} 
\usepackage{amsmath,amssymb,makecell,multicol}
\usepackage{booktabs}
\usepackage[letterpaper,top=1in,bottom=1in,left=1in,right=1in,marginparwidth=1.75cm]{geometry}
\usepackage{algorithm}
\usepackage{algpseudocode}
\usepackage{xcolor}
\usepackage{makecell}
\usepackage{array}
\usepackage{booktabs}
\usepackage{longtable}    
\usepackage{lscape}       
\usepackage{pdflscape}    
\usepackage{geometry}     

\usepackage{hyperref}
\usepackage{cleveref} 
\usepackage{float}                            
\usepackage[caption=false,font=small]{subfig} 

\newtheorem{thm}{Theorem}[section]
\newtheorem{lem}[thm]{Lemma}

\newtheorem{prop}[thm]{Proposition}
\newtheorem{assum}[thm]{Assumption}
\DeclareMathOperator{\diag}{diag}

\usepackage[parfill]{parskip}

\title{AdaGrad Meets Muon: Adaptive Stepsizes\\ for Orthogonal Updates}
\author{%
  \name Minxin Zhang \email minxinzhang@math.ucla.edu \\
  \name Yuxuan Liu \email yxliu@math.ucla.edu \\
  \name Hayden Schaeffer \email hayden@math.ucla.edu \\
  \addr Department of Mathematics \\
  University of California, Los Angeles \\
  Los Angeles, CA 90024, USA
}

\begin{document}
\newcommand{\Hess}{\nabla^2}  
\renewcommand{\Re}{\mathbb{R}} 
\newcommand{\dotP}[2]{\left\langle #1, #2 \right\rangle}
\newcommand{\tu}{\tilde{u}}
\newcommand{\tf}{\tilde{f}}
\newcommand{\abs}[1]{\left|#1\right|}
\newcommand{\norm}[1]{\left\|#1\right\|}
\newcommand{\normF}[1]{\norm{#1}_F}
\newcommand{\Set}[1]{\left\{#1\right\}}
\newcommand{\sgn}{\text{sgn}}
\newcommand{\trace}[1]{\operatorname{Tr}\left(#1\right)}
\renewcommand{\vec}[1]{\operatorname{vec}\left(#1\right)}
\newcommand{\argmin}{\operatorname*{argmin}}
\newcommand{\spa}[1]{\operatorname*{span}\left\{#1\right\}}
\newcommand{\bigO}{\mathcal O}
\renewcommand{\L}{\mathcal L}
\newcommand{\X}{\mathcal X}
\newcommand{\Z}{\mathcal Z}
\newcommand{\K}{\mathcal K}
\newcommand{\Zd}{{\scriptscriptstyle \mathcal Z}}
\newcommand{\floor}[1]{\left\lfloor #1 \right\rfloor}
\newcommand{\ceil}[1]{\left\lceil #1 \right\rceil}
\newcommand{\pr}[1]{\mathbb P\left( #1 \right)}
\newcommand{\expect}{\mathbb E}
\newcommand{\Var}[1]{\textrm{Var} \left[#1\right]}
\newcommand{\Cov}[1]{\textrm{Cov} \left[#1\right]}
\newcommand{\xstar}{x^\ast}
\newcommand{\Grad}{\nabla\!}
\newcommand{\comp}[1]{\left[#1 \right]}
\newcommand{\rank}[1]{\textrm{rank} \left(#1\right)}

\newcommand{\GD}{\textrm{\tiny GD}}
\newcommand{\OGD}{\textrm{\tiny OGD}}
\newcommand{\orth}[1]{\operatorname{Orth}\left(#1\right)}
\maketitle
\begin{abstract}%
The recently proposed Muon optimizer updates weight matrices via orthogonalized momentum and has demonstrated strong empirical success in large language model training. However, it remains unclear how to determine the learning rates for such orthogonalized updates. AdaGrad, by contrast, is a widely used adaptive method that scales stochastic gradients by accumulated past gradients. We propose a new algorithm, AdaGO, which combines a norm-based AdaGrad-type stepsize with an orthogonalized update direction, bringing together the benefits of both approaches. Unlike other adaptive variants of Muon, AdaGO preserves the orthogonality of the update direction, which can be interpreted as a spectral descent direction, while adapting the stepsizes to the optimization landscape by scaling the direction with accumulated past gradient norms. The implementation of AdaGO requires only minimal modification to Muon, with a single additional scalar variable, the accumulated squared gradient norms, to be computed, making it computationally and memory efficient. Optimal theoretical convergence rates are established for nonconvex functions in both stochastic and deterministic settings under standard smoothness and unbiased bounded-variance noise assumptions. Empirical results on CIFAR-10 classification and function regression demonstrate that AdaGO outperforms Muon and Adam.
\end{abstract}


\section{Introduction}
The trainable parameters of neural networks, including those in large language models (LLMs), are often arranged as matrices. While widely used optimization algorithms such as stochastic gradient descent (SGD), Adam \cite{kingma2014adam}, and their variants treat these parameters as flattened vectors, the recently proposed Muon optimizer \cite{jordanmuon} explicitly leverages their matrix structure. By updating weight matrices with orthogonalized momentum, Muon has demonstrated superior empirical performance \cite{liu2025muon,shah2025practical}.
Nevertheless, a fundamental question remains unresolved: what constitutes an effective learning rate for Muon? More broadly, how should one determine effective learning rates for optimizers that employ orthogonal updates?

Given a matrix $M\in\Re^{m\times n}$, its orthogonalization is defined as \[
\orth{M}:=\argmin_{O\in\Re^{m\times n}}\Set{\normF{O-M}:OO^T=I_m \textrm{ or } O^TO=I_n},
\]
where $\normF{\cdot}$ denotes the Frobenius norm. Equivalently, if $M=U\Sigma V^T$ is the reduced singular value decomposition (SVD) of $M$,
then $\orth{M}=UV^T$ \cite[Proposition 4]{bernstein2024old}. Moreover,  
\begin{equation*}
\orth{M} = -\argmin_{Z} \left[\dotP{Z}{\frac{G}{\norm{G}_*}} + \frac{1}{2} \norm{Z}_2^2\right],
\end{equation*}
where $\dotP{Z}{G}=\trace{Z^T G}$, $\norm{\cdot}_*$ denotes the nuclear norm and $\norm{\cdot}_2$ denotes the spectral norm \cite[Proposition 5]{bernstein2024old}. 
Hence, orthogonalized gradient descent (OGD) can also be interpreted as the steepest descent under the spectral norm 
\cite{fan2025implicit, carlson2015preconditioned}.
The underlying algorithm of Muon is summarized in Algorithm~\ref{alg:muon}, with the orthogonalization in Line~\ref{line:orthoM} as 
the key step distinguishing it from other SGD-based methods.
Since computing the exact orthogonalization is expensive, Muon employs Newton–Schulz iterations to obtain an efficient approximation in practice \cite{modula-docs}, whereas theoretical analyses typically assume exact orthogonalization at each iteration \cite{shen2025convergence, sato2025analysis, li2025note, chen2025muon, kovalev2025understanding, pethick2025training}. Existing convergence results of Muon generally assume a small constant learning rate; in practice, however, considerable effort is devoted to tuning the 
learning rate or designing an appropriate learning-rate schedule, as is standard for SGD-based algorithms. Yet Muon fundamentally differs from these methods, 
as its orthogonalized updates significantly alter the optimization dynamics \cite{modula-docs},
making a reconsideration of learning rate selection for such updates important.
As a motivating example, we compares the standard gradient descent (GD) and OGD in training a one-layer linear neural network in Appendix~\ref{appx:linear} .
In general, empirical observations show that when the gradient norm is large at the start of 
training, employing a large learning rate in Muon produces a rapid initial decrease in training loss but soon leads to plateauing and oscillations. In contrast, a smaller constant learning rate results in slower convergence yet ultimately achieves a lower final loss. 
These observations suggest that an adaptively tuned learning rate schedule, informed by the gradients, has the potential to further improve the 
efficiency of orthogonal updates in the Muon optimizer.

One widely known adaptive SGD method, AdaGrad, adjusts learning rates based on the cumulative history of squared gradients. 
Originally proposed in \cite{duchi2011adaptive}, the full-matrix variant of AdaGrad 
scales the update direction using the full outer product of past gradients, whereas the more practical diagonal AdaGrad 
retains only the diagonal entries of this matrix.
A more recent variant, AdaGrad-Norm \cite{ward2020adagrad}, scales the learning rate by the square root of the accumulated gradient norms.
Whereas full-matrix and diagonal AdaGrad adaptively rescale the learning rate for each parameter and thus alter the update direction, AdaGrad-Norm adjusts the stepsize through a single scalar factor while preserving the original stochastic gradient direction.
These AdaGrad stepsizes have been extensively studied in the context of standard stochastic gradients 
\cite{streeter2010less, duchi2011adaptive, faw2022power, ward2020adagrad, li2019convergence, 
kovalev2025sgd, defossez2020simple, zhou2018convergence}. However, 
their behavior under modified update directions, such as those obtained through orthogonalization, remains largely unexplored.
For instance, under the standard bounded-variance noise assumption, while the expected distance between the stochastic and true gradients is bounded, the distance between their orthogonalized counterparts can be much larger when the gradient matrices are ill-conditioned \cite{higham1986computing,zhang2025sharp}. 
Thus, orthogonalization may amplify the effect of noise in stochastic gradients, making simply accumulating gradient norms as in AdaGrad-Norm inadequate 
in this setting. In this work, we address this gap by introducing adaptive stepsizes for the orthogonalized directions. 
Motivated by the strong empirical performance of orthogonalized momentum in the Muon optimizer, we introduce a learning rate schedule that adapts to past gradients while preserving orthogonality in the updates. Specifically, we present a new algorithm, AdaGO, which combines norm-based AdaGrad stepsizes with 
orthogonalized update directions, and establish theoretical convergence guarantees for nonconvex functions.

\begin{algorithm}[!htp]
\caption{Muon}\label{alg:muon}
\begin{algorithmic}[1]
\Require Learning rate $\eta>0$, momentum $\mu\in [0,1)$, batch size $\Set{b_t}$
\State Initialize $\Theta_0\in\Re^{m\times n}$, $M_{0} = 0$
\For{$t = 1, 2, \dots, T$}
    \State Sample a minibatch of size $b_t$ and compute stochastic gradient $G_{t} = \nabla_{\Theta} \mathcal{L}_{t}(\Theta_{t-1})$
    \State $M_{t} \gets \mu M_{t-1} + (1-\mu)G_{t}$
    \State $O_{t} \gets \orth{M_{t}}$\label{line:orthoM}
    \State Update parameters $\Theta_{t} \gets\Theta_{t-1} - \eta O_t $
\EndFor
\State \Return $\Theta_{T}$
\end{algorithmic}
\end{algorithm}

\subsection{Related work}
The convergence of Muon has been analyzed in \cite{li2025note, shen2025convergence, kovalev2025understanding, pethick2025training}, 
which show that Muon converges to a stationary point at a rate of $\mathcal{O}(T^{-1/4})$ when using a constant stepsize of magnitude $\mathcal{O}(T^{-3/4})$, where 
$T$ denotes total number of iterations.
The analysis in \cite{sato2025analysis} covers four practical Muon variants, with and without Nesterov momentum and with and without weight decay, 
and derives the critical batch size.
The analysis in \cite{chen2025muon} interprets Muon as solving a spectral norm constrained problem within the Lion-K framework and establishes convergence to KKT points at a rate that depends on the batch size.

Several adaptive variants of Muon have been proposed. 
AdaMuon \cite{si2025adamuon} combines Muon with element-wise adaptivity, showing empirical improvements without theoretical convergence guarantees. COSMOS \cite{liu2025cosmos} combines SOAP \cite{vyas2024soap} and Muon for memory-efficient LLM training, reporting practical benefits in stability and memory usage, but also lacks convergence guarantees. Shampoo \cite{gupta2018shampoo} precedes Muon and is equivalent to it when momentum and accumulation are omitted.
ASGO \cite{an2025asgo} introduces an adaptive one-sided preconditioner, equivalent to Muon when momentum and accumulation are omitted. PolarGrad \cite{lau2025polargrad} unifies matrix-aware preconditioned optimizers and proposes polar-decomposition updates that subsume Muon. 
For Shampoo, ASGO, and PolarGrad, theoretical convergence has been established in convex settings.

AdaGrad stepsizes are first introduced in \cite{streeter2010less, duchi2011adaptive}. In \cite{streeter2010less}, they are shown to 
yield regret bounds typically tighter than those obtained with a fixed stepsize in online convex optimization, while \cite{duchi2011adaptive} 
extends them to stochastic optimization settings.
Unified analyses of adaptive SGD methods in convex settings are provided in \cite{kovalev2025sgd, xie2025structured}, encompassing AdaGrad, Shampoo, 
and ASGO, though not extending to Muon.
For nonconvex optimization, \cite{ward2020adagrad} establishes the convergence rate of AdaGrad-Norm stepsizes applied to stochastic gradients to 
stationary points, under the restrictive assumption that gradient norms are uniformly bounded.
Subsequently, \cite{faw2022power} relaxes this assumption, albeit at the cost of introducing higher-order polylogarithmic factors in the rate, 
and more recently, \cite{wang2023convergence} refines the analysis to recover bounds comparable to \cite{ward2020adagrad} under weaker assumptions.
High probability bounds on the convergence rates of AdaGrad for SGD are established in \cite{zhou2018convergence}, 
and \cite{defossez2020simple} analyzes its variant with momentum.
A generalized version of AdaGrad for SGD is analyzed in both convex and nonconvex settings in \cite{li2019convergence}. Moreover, 
AdaGrad-Norm has been shown to adapt to the noise level of stochastic gradients \cite{li2019convergence} and 
to the H\"older smoothness of the objective function \cite{orabona2023normalized}.

\subsection{Contributions and organization}
We propose a new algorithm, AdaGO, which combines a norm-based AdaGrad-type stepsize with an orthogonalized update direction, bringing together the benefits of Muon and AdaGrad. Unlike other adaptive variants of Muon, AdaGO preserves the orthogonality of the update direction, while adapts the stepsizes to the optimization landscape. The implementation of AdaGO requires minimal modification to Muon, with a single additional scalar variable, the accumulated squared gradient norms, to be computed, making it computationally and memory efficient.  Optimal theoretical convergence rates are established for nonconvex functions in both stochastic and deterministic settings under standard assumptions. Empirical results on CIFAR-10 classification and function regression demostate that AdaGO outperforms Muon and Adam. The rest of paper is organized as follows. We introduce the new algorithm in Section~\ref{sec:algo}, and present the theoretical analysis in
Section~\ref{sec:converge}, with proofs deferred to Appendices~\ref{appendix:sm}--\ref{appendix:sg}. 
Experimental results are reported in Section~\ref{sec:experiments}, and Section~\ref{sec:conclude} concludes with a discussion of future directions.
\section{AdaGO: A New Algorithm}\label{sec:algo}
In this section, we present the new algorithm, AdaGO, combining stepsizes adaptively tuned by past gradients with orthogonalized updates. 
The details of AdaGO are summarized in Algorithm~\ref{alg:adaGO}.

\begin{algorithm}[H]
\caption{AdaGO}\label{alg:adaGO}
\begin{algorithmic}[1]
\Require Learning rate $\eta>0$, momentum $\mu\in [0,1)$, batch size $\Set{b_t}$, $\gamma>0$, $\epsilon>0$
\State Initialize $\Theta_0\in\Re^{m\times n}$, $M_{0} = 0$, $v_0>0$
\For{$t = 1, 2, \dots, T$}
    \State Sample a minibatch of size $b_t$ and compute stochastic gradient $G_{t} = \nabla \mathcal{L}_{t}(\Theta_{t-1})$
    \State $M_{t} \gets \mu M_{t-1} + (1-\mu)G_{t}$
    \State $v_{t}^2 \gets v_{t-1}^2+\min\{\norm{G_t}^2,\gamma^2\}$\label{line:accumG}
    \State $O_{t} \gets \orth{M_{t}}$
    \State Update parameters $\Theta_{t} \gets \Theta_{t-1} \;-\; \max\{\epsilon,\eta\frac{\min\{\norm{G_t},\gamma\}}{v_{t}}\}\,O_{t} $\label{line:update}
\EndFor
\State \Return $\Theta_{T}$
\end{algorithmic}
\end{algorithm}

At each iteration, AdaGO updates the training parameters by \[
\Theta_t = \Theta_{t-1} -\alpha_t O_{t}, \quad\textrm{ with }~ \alpha_t:= \max\left\{\epsilon,\eta\frac{\min\{\norm{G_t},\gamma\}}{v_{t}}\right\},
\]
where $O_t$ is the orthogonalized momentum and $\alpha_t$ is an adaptive stepsize.
Recall that AdaGrad-Norm accumulates the squared norms of past gradients and scales the stochastic gradient by the reciprocal of the square root of this accumulation, with its convergence rate established under the relatively restrictive assumption that the gradient norms are uniformly bounded \cite{ward2020adagrad}.
For AdaGO, we remove this assumption and instead accumulate the squared norms clamped by a large constant $\gamma > 0$ to obtain $$
v_t^2 = \sum_{\tau=0}^t \min\{\norm{G_\tau}^2,\gamma^2\}.$$ Empirically, AdaGO performs robustly across a wide range of $\gamma$ values. To prevent numerical instability from division by small denominators in the stepsize computation, we initialize the accumulator $v_t$ with $v_0>0$. Theoretically, Section~\ref{sec:converge} shows that $\gamma$ and $v_0$ appear only in logarithmic terms in the convergence error bounds, and thus have limited impact on performance. Moreover, since the orthogonalized momentum has unit magnitude, we scale $O_t$ by the clamped current gradient norm, i.e., $\min\{\norm{G_t},\gamma\}$. This ensures that the per-iteration update decays to zero as AdaGO converges to a stationary point—a property known 
as \emph{null gradient consistency}, which is generally desirable for optimization algorithms \cite{lau2025polargrad}.
In addition, we impose a lower bound $\epsilon>0$ on the stepsizes, thereby ensuring that AdaGO converges at least as fast as 
Muon with a small constant stepsize. As shown in the analysis in Section~\ref{sec:converge}, the choice of $\epsilon$ depends
on the optimization stopping time $T.$ The theoretical results in the following section hold for any choice of matrix norm for the gradients. 
In practice, however, we use the Frobenius norm of $G_t$ in Lines~\ref{line:accumG} and \ref{line:update} of Algorithm~\ref{alg:adaGO} for computational efficiency.

\section{Convergence Analysis}\label{sec:converge}
For the convergence analysis of AdaGO, we impose the standard assumptions that the loss function $\L(\Theta)$ is $L$-smooth and that the stochastic gradient is an unbiased estimator of the true gradient with bounded variance.
\begin{assum}\label{assum:func}
The gradient of $\L(\Theta)$ is Lipschitz continuous, i.e., for arbitrary $\Theta, \Theta'\in\Re^{m\times n},$
\begin{equation}\label{eq:lipG}
\norm{\Grad\L(\Theta)-\Grad\L(\Theta')}_*\le L\norm{\Theta-\Theta'}_2
\end{equation}
for some constant $L>0,$ where $\norm{\cdot}_*$ and $\norm{\cdot}_2$ denote the nuclear norm and the spectral norm respectively.
\end{assum}
\begin{assum}\label{assum:noise}
At each iteration $t,$ the stochastic gradient $G_t$ is an unbiased estimate of the true gradient, i.e.,
$\expect[G_t] = \Grad\L(\Theta_{t-1}),$
with a uniformly bounded variance
\[\expect\left[\normF{G_t-\Grad\L(\Theta_{t-1})}^2\right]\le \frac{\kappa^2}{b_t},\]
where $b_t\ge 1$ is the batch size and $\normF{\cdot}$ denotes the Frobenius norm.
\end{assum}
Note that Assumption~\ref{assum:func} is equivalent to a more commonly used assumption:
\begin{equation}\label{eq:lipGF}
\normF{\Grad\L(\Theta)-\Grad\L(\Theta')}\le L'\normF{\Theta-\Theta'}
\end{equation}
for a different Lipschitz constant $L'>0$. Since OGD is interpreted as the steepest descent under the spectral norm,
we assume Eq. \eqref{eq:lipG} for the analysis of AdaGO. A detailed discussion on the two equivalent assumptions are 
presented in \cite{shen2025convergence}.

The convergence of AdaGO is established in the following theorem, with the proof provided in Appendix~\ref{appendix:sm}. 
\begin{thm}\label{thm:sm}
Suppose Assumptions~\ref{assum:func}--\ref{assum:noise} holds. Let $\Set{\Theta_t}\subset\Re^{m\times n}$ be the sequence of iterates generated by
Algorithm~\ref{alg:adaGO} and write $\Delta:=\L(\Theta_0)-\min_{\Theta}\L(\Theta)$ and $r:=\min\{m,n\}.$ If we set $b_t\equiv1$, $\epsilon=T^{-\frac{3}{4}}$,
$1-\mu=T^{-\frac{1}{2}}$, and $\eta= T^{-\left(\frac{3}{8}+q\right)}$ for arbitrary $q>0$, then, for large $T,$ \[
\frac{1}{T}\sum_{t=1}^T \expect\left[\norm{\Grad\L(\Theta_{t-1})}_*\right]\\
\le  \bigO\left(\frac{\Delta+\kappa\sqrt{r}+L}{T^{\frac{1}{4}}}+\frac{L\sqrt{r}}{T^{\frac{1}{4}+q}}\left(\ln\left(\frac{\gamma^2}{v_0^2}T\right)+1\right)\right).
\]
\end{thm}
By \cite[Theorem 3]{arjevani2023lower},
the $\mathcal{O}(T^{-1/4})$ rate established above is the best possible convergence rate for stochastic first-order methods under Assumptions~\ref{assum:func}--\ref{assum:noise}.

We also establish the convergence of AdaGO in the deterministic setting without momentum in the following theorem, with the proof given in Appendix~\ref{appendix:determin}. 
\begin{thm}\label{thm:determin}
Suppose Assumptions~\ref{assum:func}--\ref{assum:noise} holds. Let $\Set{\Theta_t}\subset\Re^{m\times n}$ be the sequence of iterates generated by
Algorithm~\ref{alg:adaGO} using full batch with $\mu=0$.
Write $\Delta:=\L(\Theta_0)-\min_{\Theta}\L(\Theta)$ and $r:=\min\{m,n\}.$ 
If $\epsilon = T^{-\frac{1}{2}}$ and $\eta=T^{-q}$ for arbitrary $q>0,$ then, for large $T$,
\[
\frac{1}{T}\sum_{t=1}^T \expect\left[\norm{\Grad\L(\Theta_{t-1})}_*\right]\le \bigO\left(\frac{\Delta+L}{\sqrt{T}}\right).
\]
\end{thm}
As shown in \cite[Theorem 2]{carmon2020lower}, the $\bigO(1/\sqrt{T})$ rate established above is the best possible 
convergence rate for deterministic first-order methods under Assumption~\ref{assum:func}.

Of theoretical interest, we further analyze the behavior of AdaGO in the stochastic setting when momentum is turned off. 
The following theorem shows that Algorithm~\ref{alg:adaGO} without momentum converges if the batch size $b_t$ increases as $t$ increases,
with the proof given in Appendix~\ref{appendix:sg}. 
\begin{thm}\label{thm:sg}
Suppose Assumptions~\ref{assum:func}--\ref{assum:noise} holds. Let $\Set{\Theta_t}\subset\Re^{m\times n}$ be the sequence of iterates generated by
Algorithm~\ref{alg:adaGO} with the momentum $\mu=0$.
If we set the batch size $b_t=\sqrt{t}$, $\epsilon = T^{-\frac{1}{2}}$ and $\eta=T^{-q}$ for 
arbitrary $q>0,$ then for large $T,$
\[
\frac{1}{T}\sum_{t=1}^T\expect\left[\norm{\Grad\L(\Theta_{t-1})}_*\right]
\le \bigO\left(\frac{\kappa\sqrt{r}}{T^{\frac{1}{4}}}+\frac{L+\Delta}{\sqrt{T}}\right).
\]
Or, if $b_t=t$, $\epsilon = T^{-\frac{1}{2}}$ and $\eta=T^{-q}$ for 
arbitrary $q>0,$ then
\[
\frac{1}{T}\sum_{t=1}^T\expect\left[\norm{\Grad\L(\Theta_{t-1})}_*\right]
\le \bigO\left(\frac{\kappa\sqrt{r}+L+\Delta}{\sqrt{T}}\right).
\]
\end{thm}
This result implies that AdaGO adapts to the noise level of stochastic gradients, 
which is consistent with the behavior of AdaGrad stepsizes for SGD shown in \cite{li2019convergence}.

\section{Experiments}\label{sec:experiments}

\subsection{Experiment Setup}

\paragraph{Baselines.}
We compare our proposed optimizer, AdaGO, against two strong baselines: Adam \cite{kingma2014adam} and Muon \cite{jordan2024muon}. Adam is a widely used default optimizer for deep learning, while Muon has recently demonstrated strong empirical performance \cite{chen2025muon,shah2025practical}. 
Since Muon and AdaGO are designed specifically for matrix parameters, we use Adam to optimize all scalar and vector parameters in the models. 
For brevity, we refer to these hybrid methods simply as Muon and AdaGO.

In our experiments, we use standard hyperparameter settings for the baselines. For Adam, we set the momentum coefficients to $\beta_1 = 0.9$ and $\beta_2 = 0.95$. For Muon and AdaGO, the momentum coefficient is set to $\beta=0.95$. We perform a grid search to find the optimal learning rate $\eta$ for each optimizer on each task. For AdaGO, we also tune the $\epsilon$ hyperparameter. Weight decay is not used.

\paragraph{Datasets and Models.}
We evaluate the optimizers on two tasks: function regression and image classification on CIFAR-10.

For the function regression task, we generate a dataset by sampling 10,000 points from a Gaussian random field with 50-dimensional input and 50-dimensional output (10\% are used as testing data). We use a two-layer MLP with GeLU activation and a hidden dimension of 100 to fit the data. The model is trained for 1000 steps using the mean squared error loss.

For CIFAR-10 classification, we use a convolutional neural network consisting of 3 convolutional layers and 2 fully connected layers. We train the model for 100 epochs using a batch size of 128 and the standard cross-entropy loss. We report both the training loss and the test accuracy.

\subsection{Results}

The optimal hyperparameters found through our grid search are summarized in Table~\ref{tab:hyper}. 
Although AdaGO introduces an additional hyperparameter $\epsilon$, both theoretical analysis and 
empirical observations indicate that its choice is guided by the value of $\eta$; specifically,
an effective $\epsilon$ satisfies $\epsilon < \eta^2$.
The performance of each optimizer on the two tasks is detailed below.

\begin{table}
\centering
\begin{tabular}{c|c|c}
\hline
Optimizer & Regression Hyperparameters & Classification Hyperparameters \\
\hline
Adam & $\eta=0.01$ & $\eta=3\times 10^{-4}$ \\
Muon & $\eta=5\times 10^{-3}$ & $\eta=2\times 10^{-3}$ \\
AdaGO & $\eta=0.5,~\epsilon=5\times 10^{-3}$ & $\eta=5\times 10^{-2},~\epsilon=5\times 10^{-4}$ \\
\hline
\end{tabular}
\caption{Optimal hyperparameters for different optimizers and tasks.}
\label{tab:hyper}
\end{table}

\begin{figure}[t]
  \centering
  \subfloat[Training loss\label{fig:reg-train-loss}]{
    \includegraphics[width=0.48\linewidth]{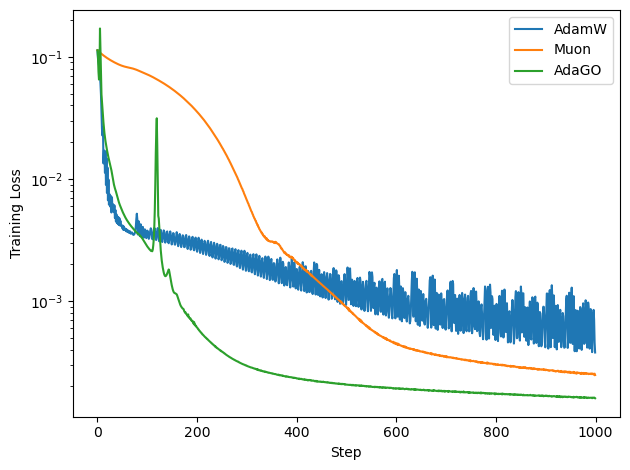}
  }
  \hfill
  \subfloat[Test accuracy\label{fig:reg-test-acc}]{
    \includegraphics[width=0.48\linewidth]{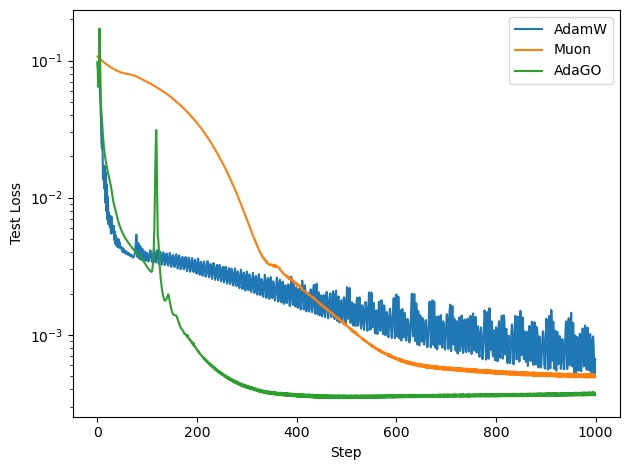}
  }

  \caption{Comparing optimizer performance for the regression task.}
  \label{fig:reg}
\end{figure}

\begin{figure}[t]
  \centering
  \subfloat[Training loss\label{fig:train-loss}]{
    \includegraphics[width=0.48\linewidth]{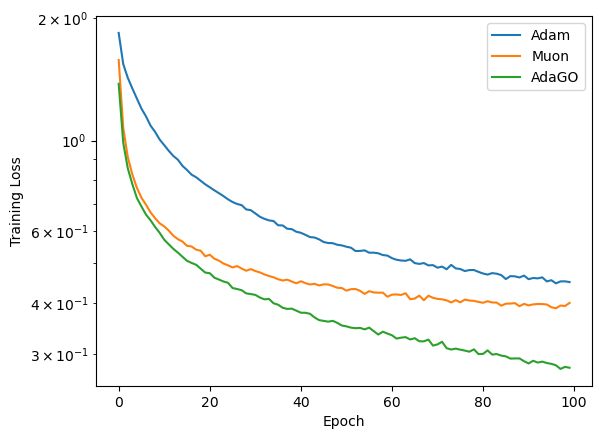}
  }
  \hfill
  \subfloat[Test accuracy\label{fig:test-acc}]{
    \includegraphics[width=0.45\linewidth]{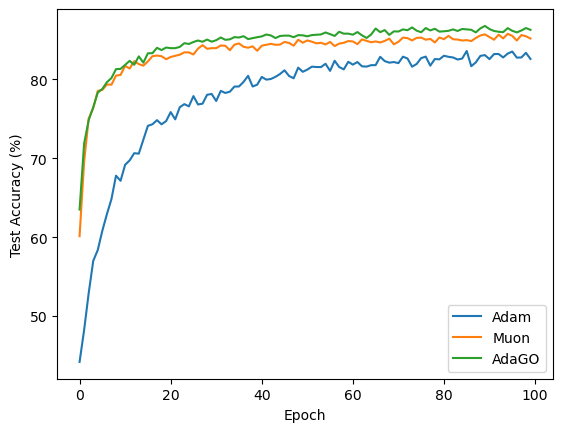}
  }
  \caption{Comparing the performance of optimizers on the CIFAR-10 image classification task.}
  \label{fig:class}
\end{figure}

For the function regression task, we plot the training and test loss curves in Figure~\ref{fig:reg}. 
AdaGO demonstrates superior performance, converging to lower final training and test losses than Adam and Muon.
The plot shows that Adam's convergence is hindered by significant oscillations, 
a common issue with higher learning rate. Muon provides a more stable descent but converges to a higher loss value than AdaGO. While AdaGO's loss curve exhibits some sharp spikes, indicating aggressive updates, it recovers quickly from these perturbations and continues to make progress, highlighting its robust optimization capability.

On the CIFAR-10 classification task, AdaGO's advantages are also evident, as shown in Figure~\ref{fig:class}. Throughout the 100 epochs, AdaGO consistently maintains a lower training loss than the baselines (Figure~\ref{fig:class}a). More importantly, this improved optimization translates to better generalization. In Figure~\ref{fig:class}b, we see that AdaGO achieves higher test accuracy than both Muon and Adam. These results suggest that AdaGO not only accelerates training but also guides the model to a solution that generalizes more effectively.

In summary, across both regression and classification tasks, AdaGO consistently outperforms Adam and Muon, experimentally demonstrating improved performance.

\section{Conclusions and Future Work}\label{sec:conclude}
In this work, we propose AdaGO, a new optimizer that combines a norm-based AdaGrad-type stepsize with an orthogonalized update direction, bringing together
the benefits of both Muon and AdaGrad. Unlike other adaptive variants of Muon, AdaGO preserves the orthogonality of the update directions while adapting stepsizes to the optimization landscape. Its implementation requires only minimal modification to Muon, with a single additional scalar variable, the accumulated squared
gradient norms, to be computed, making it both computationally and memory efficient. We establish optimal convergence rates for nonconvex functions in
both stochastic and deterministic settings under standard assumptions. Experimental results on CIFAR-10 classification and function regression tasks 
demonstrate the consistent improved performance of AdaGO over Muon and Adam. 
Future work includes testing AdaGO on LLM training, analyzing the algorithm under relaxed assumptions, and incorporating practical enhancements, thereby further advancing adaptive strategies for orthogonalized updates.

\section*{Acknowledgement}
\paragraph{Funding.} This work was supported in part by NSF DMS 2331033.
\bibliography{ref}

\appendix
\section{Motivating Example: GD vs. OGD in a Linear Case}\label{appx:linear}
As a motivating example, we compare GD and OGD in training a one-layer linear neural network.
The loss function is given by \[
\L(W) :=\frac{1}{2}\sum_{j=1}^J\norm{Wx_j-y_j}^2,
\]
where $W\in \Re^{m\times d},$ and $\Set{(x_j,y_j}_{j=1}^J$ are training data. 
For simplicity, assume that $X:=[x_1,\cdots,x_J]\subset\Re^{d\times J}$ is of full rank, and the loss function has a unique minimizer $W_*$ 
such that $\L(W_*)=0.$
The gradient of $\L$ is given by \begin{align*}
\Grad\L(W) = & \sum_{j=1}^J\left(Wx_j-y_j\right)x_j^T\nonumber\\
= & \sum_{j=1}^J\left(W-W_*\right) x_jx_j^T\nonumber\\
= & (W-W_*)XX^T.
\end{align*}
Let $\Set{W_t^{\GD}}$ and $\Set{W_t^{\OGD}}$ 
be sequences of iterates generated by GD and OGD respectively, with a learning rate schedule $\Set{\eta_t}.$
For GD, it holds that
\begin{align}\label{eq:linearGD}
\normF{W_{t+1}^{\GD}-W_*} = & \normF{W_{t}^{\GD}-\eta_t\Grad\L\left(W_{t}^{\GD}\right) -W_*}& \nonumber\\
=& \normF{\left(W_t^{\GD}-W_*\right)\left(I-\eta_t XX^T\right)}\nonumber\\
\le&  \norm{I-\eta_tXX^T}_2\normF{W_t^{\GD}-W_*},
\end{align}
where
$\normF{\cdot}$ denotes the Frobenius norm.
On the other hand, for OGD, let \[
\Grad\L( W_{t}^{\OGD}) =U_t\Sigma_t V_t^T
\]
be the reduced SVD for each $t.$ Then \[
\orth{\Grad\L( W_{t}^{\OGD})} = U_tV_t^T = \Grad\L( W_{t}^{\OGD})\left(V_t\Sigma_t^{-1} V_t^T\right)= \Grad\L( W_{t}^{\OGD}) P_t^{-1},
\]
where $$P_t:= V_t\Sigma_t V_t^T+\norm{\Grad\L\left( W_{t}^{\OGD}\right)}_2 \left(I-V_tV_t^T\right)$$ is a positive definite matrix.
It follows that 
\begin{align}\label{eq:linearOGD}
\normF{W_{t+1}^{\OGD}-W_*} = & \normF{W_{t}^{\OGD}-\eta_t\Grad\L( W_{t}^{\OGD}) P_t^{-1} -W_*}\nonumber\\
= & \normF{\left(W_{t}^{\OGD} -W_*\right)\left(I-\eta_t XX^TP_t^{-1}\right)}\nonumber\\
\le & \norm{I-\eta_t XX^TP_t^{-1}}_2\normF{W_t^{\OGD}-W_*},
\end{align}
with $P_t$ having eigenvalues between the largest and the smallest nonzero singular values of $\Grad \L(W_t).$

By Eq. \eqref{eq:linearGD}, the standard GD with a sufficiently small constant learning rate contracts the distance between $W_t^{\GD}$ and $W_*$ at a linear rate.
In contrast, by Eq. \eqref{eq:linearOGD}, when $\norm{\Grad\L(W_t^{\OGD})}$ is large at the start of training, a larger $\eta_t$ enables faster convergence in OGD.
As $\norm{\Grad\L(W_t^{\OGD})}$ decreases toward zero, however, $\eta_t$ must be reduced to ensure convergence of $W_t^{\OGD}$ to $W_*$.
As illustrated in Figure~\ref{fig:linear}, for OGD with a constant learning rate $\eta_t \equiv \eta$, a large $\eta$ initially drives rapid reduction in the training loss but soon leads to plateauing and oscillation, whereas a smaller $\eta$ results in slower convergence but ultimately achieves a lower final loss. Similar behavior is observed when training nonlinear networks (Figure~\ref{fig:oneAttention}).
Therefore, a learning rate schedule adaptively tuned to gradient norms is desirable.

\begin{figure}[htbp]\label{fig:motive}
  \centering
  \subfloat[One-layer linear network \label{fig:linear}]{%
    \includegraphics[width=0.47\textwidth]{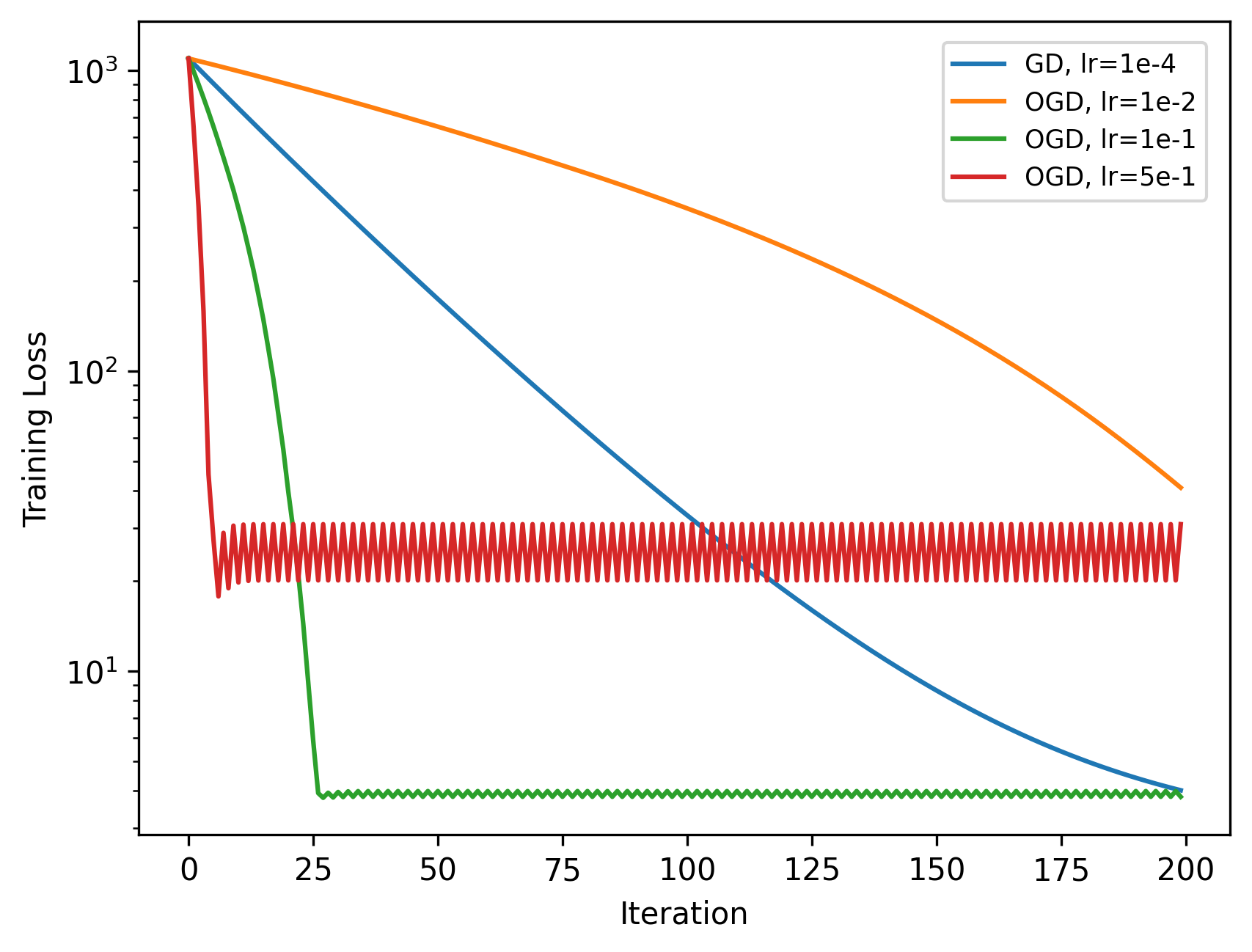}%
  }\hfill
  \subfloat[One-layer attention \label{fig:oneAttention}]{%
    \includegraphics[width=0.47\textwidth]{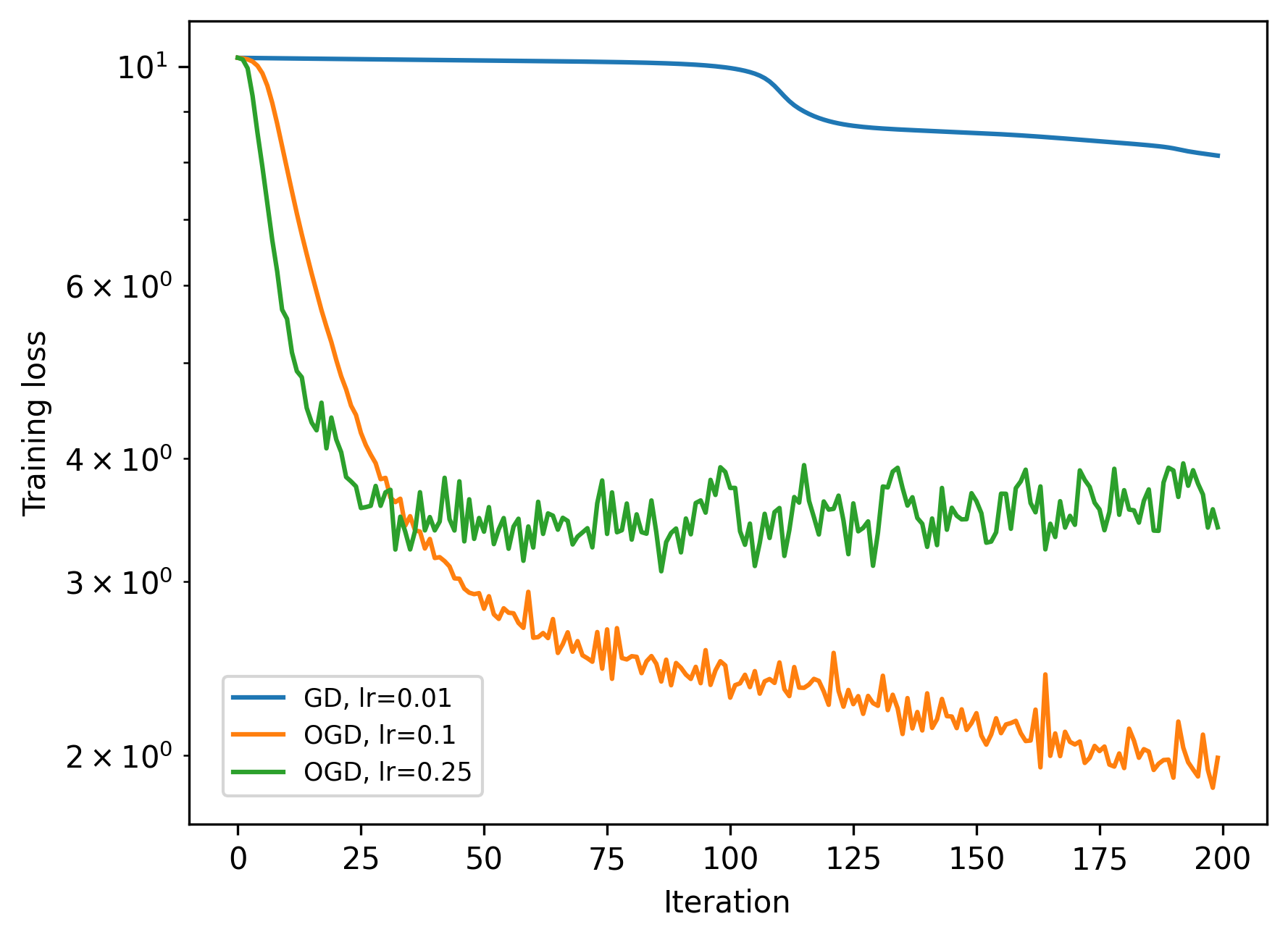}%
  }
  \caption{GD vs OGD with varied learning rates (lr)}
\end{figure}

\section{Useful Lemmas}
\begin{lem}\label{lemma:lips}
The Lipschitz continuity of $\Grad \L$ in Assumption~\ref{assum:func} implies
\begin{equation}\label{eq:smooth}
\L(\Theta')\le \L(\Theta)+\dotP{\Grad\L(\Theta)}{\Theta'-\Theta}+\frac{L}{2}\norm{\Theta'-\Theta}_2^2.
\end{equation}
\end{lem}
\begin{proof}
For $s\in [0,1]$, define $h(s):= \L(\Theta+s(\Theta'-\Theta))$. Then
\begin{align*}
\L(\Theta')-\L(\Theta)&=\int_{0}^{1}h'(s) ds\\
&=\int_{0}^{1}\dotP{\Grad \L(\Theta+s(\Theta'-\Theta))}{\Theta'-\Theta}ds\\
&=\dotP{\Grad\L(\Theta)}{\Theta'-\Theta}+\int_{0}^{1}\dotP{\Grad \L(\Theta+s(\Theta'-\Theta))-\Grad\L(\Theta)}{\Theta'-\Theta}ds\\
&\le \dotP{\Grad\L(\Theta)}{\Theta'-\Theta}+\int_{0}^{1}\norm{\Grad \L(\Theta+s(\Theta'-\Theta))-\Grad\L(\Theta)}_*\norm{\Theta'-\Theta}_2 ds \\
&\le\dotP{\Grad\L(\Theta)}{\Theta'-\Theta}+\frac{L}{2}\norm{\Theta'-\Theta}_2^2.
\end{align*}
\end{proof}

\begin{lem}\label{lem:as}
For arbitrary nonnegative values, $\Set{a_t}_{t=1}^T,$ with $a_1>0,$ it holds that \begin{equation}\label{eq:as}
\sum_{t=1}^T\frac{a_t}{\sum_{\tau=1}^t a_\tau}\le \ln\left(\sum_{t=1}^T \frac{a_t}{a_1}\right)+1.
\end{equation}
\end{lem}
\begin{proof}
A similar result is shown in \cite[Lemma 3.2]{ward2020adagrad}. For completeness, we include the proof here.
Write $S_t:=\sum_{\tau=1}^t a_\tau.$ We first show that \begin{equation*}
\frac{a_t}{S_t} \le \ln(S_t)-\ln(S_{t-1})
\end{equation*}
for all $t\ge 2.$ Indeed, by the Mean Value Theorem, there exists $\xi_t\in [S_{t-1},S_t]$
such that \[\ln(S_t)-\ln(S_{t-1})=\frac{S_t-S_{t-1}}{\xi_t}=\frac{a_t}{\xi_t}\ge\frac{a_t}{S_t}.\]
Hence, \[
\sum_{t=1}^T \frac{a_t}{S_t}\le 1+\sum_{t=2}^T \left(\ln(S_t)-\ln(S_{t-1})\right)=1+\ln(S_T)-\ln(S_1)=1+\ln\left(\frac{S_T}{a_1}\right),
\]
which is exactly Eq. \eqref{eq:as}.
\end{proof}


\section{Proof of Theorem~\ref{thm:sm}}\label{appendix:sm}
\begin{proof}
Write \[
\alpha_t:= \frac{\min\{\norm{G_t},\gamma\}}{v_t}.
\]
Let $\expect_t[\cdot]:=\expect[\cdot|\Theta_{t-1}]$ denote the conditional expectation given the previous iterates
$\Theta_0,\cdots,\Theta_{t-1}.$ 
By Lemma~\ref{lemma:lips},
\begin{align*}
&\expect_t\left[\L(\Theta_{t})-\L(\Theta_{t-1})\right]\\
\le &\expect_t\left[-\dotP{\Grad\L(\Theta_{t-1})}
{ \max\left\{\epsilon,\eta\alpha_t\right\} O_t}\right]+  \frac{L}{2}\expect_t\left[ \max\left\{\epsilon,\eta\alpha_t\right\}^2\right]\\
= & \expect_t\left[-\dotP{\Grad\L(\Theta_{t-1})-M_t}{ \max\left\{\epsilon,\eta\alpha_t\right\} O_t}\right] 
-\expect_t\left[ \max\left\{\epsilon,\eta\alpha_t\right\}\norm{M_t}_*\right] +\frac{L\epsilon^2}{2}+
\frac{\eta^2 L}{2}\expect_t\left[ \alpha_t^2\right]\\
\le & \expect_t\left[ \max\left\{\epsilon,\eta\alpha_t\right\}\norm{\Grad\L(\Theta_{t-1})-M_t}_*\right] 
-\expect_t\left[ \max\left\{\epsilon,\eta\alpha_t\right\}\norm{M_t}_*\right] 
+\frac{L\epsilon^2}{2}+\frac{\eta^2 L}{2}\expect_t\left[ \alpha_t^2\right]\\
\le & \expect_t\left[2 \max\left\{\epsilon,\eta\alpha_t\right\}\norm{\Grad\L(\Theta_{t-1})-M_t}_*\right] 
-\expect_t\left[ \max\left\{\epsilon,\eta\alpha_t\right\}\norm{\Grad\L(\Theta_{t-1})}_*\right] 
+\frac{L\epsilon^2}{2}+\frac{\eta^2 L}{2}\expect_t\left[ \alpha_t^2\right].
\end{align*}
Then by the law of total expectation,
\begin{align}\label{eq:totalM}
&\sum_{t=1}^T \expect\left[ \max\left\{\epsilon,\eta\alpha_t\right\}\norm{\Grad\L(\Theta_{t-1})}_*\right]\nonumber\\
\le &  \L(\Theta_0)-\L_\star+2\sum_{t=1}^T\expect\left[
 \max\left\{\epsilon,\eta\alpha_t\right\}\norm{\Grad\L(\Theta_{t-1})-M_t}_*\right]
+\frac{L\epsilon^2 T}{2}+\frac{\eta^2 L}{2}\sum_{t=1}^T \expect\left[ \alpha_t^2\right]\nonumber\\
\le & \L(\Theta_0)-\L_\star+2\epsilon\sum_{t=1}^T\expect\left[\norm{E_t}_*\right]
+2\eta\sum_{t=1}^T\expect\left[\alpha_t \norm{E_t}_*\right]+
\frac{L\epsilon^2 T}{2}+\frac{\eta^2 L}{2}\sum_{t=1}^T\expect\left[ \alpha_t^2\right],
\end{align}
where $E_t:=M_t-\Grad\L(\Theta_{t-1})$ and $\L_*:=\min_{\Theta}\L(\Theta)$.

By Cauchy-Schwarz inequality and Lemma~\ref{lem:as}, the third term on the right side of the above inequality satisfies
\begin{align}\label{eq:csM}
\sum_{t=1}^T\expect\left[ \alpha_t\norm{E_t}_*\right]
\le &\sum_{t=1}^T\sqrt{\expect\left[\norm{E_t}_*^2\right]\expect\left[
\alpha_t^2\right]}\nonumber\\
\le & \sqrt{\sum_{t=1}^T \expect\left[\norm{E_t}_*^2\right]}
\sqrt{\sum_{t=1}^T \expect\left[ \alpha_t^2\right]}\nonumber\\
\le & \sqrt{\sum_{t=1}^T \expect\left[\norm{E_t}_*^2\right]}\sqrt{\ln\left(\frac{\gamma^2}{v_0^2}T\right)+1}.
\end{align}
Now write $\tilde E_t:=G_t-\Grad\L(\Theta_{t-1})$ for $t\ge 1.$
Then \begin{align*}
M_{t+1} =&  \mu M_t+(1-\mu)G_{t+1}\\
= & \mu \left(E_t+\Grad\L(\Theta_{t-1})\right)+(1-\mu)\left(\tilde E_{t+1}+\Grad\L(\Theta_t)\right)\\
= & \Grad\L(\Theta_t)+\mu\left(\Grad\L(\Theta_{t-1})-\Grad\L(\Theta_t)\right)+\mu E_t+(1-\mu)\tilde E_{t+1}.
\end{align*}
Hence, for $t\ge 0,$
\[
E_{t+1} = \mu\left(\Grad\L(\Theta_{t-1})-\Grad\L(\Theta_t)\right)+\mu E_t+(1-\mu)\tilde E_{t+1}.
\]
A recursive formula can be derived as in the proof of \cite[Theorem]{cutkosky2020momentum}:
\begin{equation}\label{eq:recursiveE}
E_{t+1} = \mu^t E_1+(1-\mu)\sum_{\tau=0}^{t-1}\mu^\tau \tilde E_{t+1-\tau}+
\mu\sum_{\tau=0}^{t-1}\mu^\tau\left(\Grad\L(\Theta_{t-\tau-1})-\Grad\L(\Theta_{t-\tau})\right).
\end{equation}
Assuming a minibatch of size $b>0$ is sampled independently at each iteration, it follows that
\begin{align*}
&\expect\left[\norm{E_{t+1}}_*\right]\le \mu^t \expect\left[\norm{E_1}_*\right]+(1-\mu)\expect\left[\norm{\sum_{\tau=0}^{t-1}\mu^\tau 
\tilde E_{t+1-\tau}}_*\right]+
\mu L\sum_{\tau=0}^{t-1}\mu^\tau \expect\left[\max\{\epsilon,\eta\alpha_{t-\tau}\}\right]\\
\le & \frac{\mu^t\kappa\sqrt{r}}{\sqrt{b}}+(1-\mu)\sqrt{r}\expect\left[\normF{\sum_{\tau=0}^{t-1}\mu^\tau 
\tilde E_{t+1-\tau}}\right]+\mu L\epsilon\frac{1-\mu^t}{1-\mu}+
\mu\eta L\sum_{\tau=0}^{t-1}\mu^\tau \expect\left[\alpha_{t-\tau}\right]\\
\le &  \frac{\mu^t\kappa\sqrt{r}}{\sqrt{b}}+(1-\mu)\sqrt{r}\expect\left[\normF{\sum_{\tau=0}^{t-1}\mu^\tau 
\tilde E_{t+1-\tau}}^2\right]^{\frac{1}{2}}+\mu L\epsilon\frac{1-\mu^t}{1-\mu}+
\mu\eta L\sum_{\tau=0}^{t-1}\mu^\tau  \expect\left[\alpha_{t-\tau}\right]\\
\le &  \frac{\mu^t\kappa\sqrt{r}}{\sqrt{b}}+(1-\mu)\frac{\kappa\sqrt{r}}{\sqrt{b}}\sqrt{\sum_{\tau=0}^{t-1}\mu^{2\tau} }+\mu L\epsilon\frac{1-\mu^t}{1-\mu}+
\mu\eta L\sum_{\tau=0}^{t-1}\mu^\tau  \expect\left[\alpha_{t-\tau}\right]\\
= & \frac{\mu^t\kappa\sqrt{r}}{\sqrt{b}}+(1-\mu)\frac{\kappa\sqrt{r}}{\sqrt{b}}\sqrt{\frac{1-\mu^{2t}}{1-\mu^2}}+\mu L\epsilon\frac{1-\mu^t}{1-\mu}+
\mu\eta L\sum_{\tau=0}^{t-1}\mu^\tau  \expect\left[\alpha_{t-\tau}\right].
\end{align*}
Therefore,
\begin{align}\label{eq:errorN}
\sum_{t=0}^T\expect\left[\norm{E_{t+1}}_*\right]\le & \frac{\kappa\sqrt{r}}{(1-\mu)\sqrt{b}}+T\kappa\sqrt{\frac{r(1-\mu)}{b}}+\frac{T\mu L\epsilon}{1-\mu}+
\mu\eta L\sum_{t=1}^T\sum_{\tau=0}^{t-1}\mu^\tau  \expect\left[\alpha_{t-\tau}\right]\nonumber\\
\le & \frac{\kappa\sqrt{r}}{(1-\mu)\sqrt{b}}+T\kappa\sqrt{\frac{r(1-\mu)}{b}}+\frac{T\mu L\epsilon}{1-\mu}+\frac{\mu\eta L}{1-\mu}\sum_{t=1}^T
\expect\left[\alpha_{t}\right]\nonumber\\
\le & \frac{\kappa\sqrt{r}}{(1-\mu)\sqrt{b}}+T\kappa\sqrt{\frac{r(1-\mu)}{b}}+\frac{T\mu L\epsilon}{1-\mu}+\frac{\mu\eta L}{1-\mu}\sqrt{T}\left(\sum_{t=1}^T
\expect\left[\alpha_{t}^2\right]\right)^{\frac{1}{2}}\nonumber\\
\le & \frac{\kappa\sqrt{r}}{(1-\mu)\sqrt{b}}+T\kappa\sqrt{\frac{r(1-\mu)}{b}}+\frac{T\mu L\epsilon}{1-\mu}+\frac{\mu\eta L}{1-\mu}\sqrt{T}\left(\ln\left(\frac{\gamma^2}{v_0^2}T\right)+1\right).
\end{align}
Also, by Eq. \eqref{eq:recursiveE},
\begin{align*}
\expect\left[\normF{E_{t+1}}^2\right]=&\mu^{2t}\expect\left[\normF{E_1}^2\right]+\mu^{t+1}
\sum_{\tau=0}^{t-1}\mu^\tau\expect\left[\dotP{E_1}{\Grad\L(\Theta_{t-\tau-1})-\Grad\L(\Theta_{t-\tau})}\right]\\
&+(1-\mu)^2\sum_{\tau=0}^{t-1}\mu^{2\tau} \expect\left[\normF{\tilde E_{t+1-\tau}}^2\right]+
\mu^2\expect\left[\normF{\sum_{\tau=0}^{t-1}\mu^\tau\left(\Grad\L(\Theta_{t-\tau-1})-\Grad\L(\Theta_{t-\tau})\right)}^2\right]\\
\le& \frac{\mu^{2t}\kappa^2}{b}+ L\mu^{t+1}
\sum_{\tau=0}^{t-1}\mu^\tau\expect\left[\normF{E_1}\max\{\epsilon,\eta\alpha_{t-\tau}\}\right]\\
&+(1-\mu)^2\frac{\kappa^2}{b}\sum_{\tau=0}^{t-1}\mu^{2\tau}+
\mu^2\expect\left[\left(\sum_{\tau=0}^{t-1}\mu^\tau\normF{\Grad\L(\Theta_{t-\tau-1})-\Grad\L(\Theta_{t-\tau})}\right)^2\right]\\
\end{align*}
Then by Cauchy-Schwarz inequality,
\begin{align}\label{eq:errorM}
\expect\left[\normF{E_{t+1}}^2\right]\le & \frac{\mu^{2t}\kappa^2}{b}+\frac{ L\kappa\epsilon\mu^{t+1}}{\sqrt{b}}\frac{1-\mu^t}{1-\mu}
+\frac{\eta L\kappa\mu^{t+1}}{\sqrt{b}}\sum_{\tau=0}^{t-1}\mu^\tau
\sqrt{\expect\left[\alpha_{t-\tau}^2\right]}+
\frac{\kappa^2}{b}\frac{(1-\mu)(1-\mu^{2t})}{1+\mu}\nonumber\\
&+ L^2\mu^2\epsilon^2\left(\frac{1-\mu^{2t}}{1-\mu^2}\right) +\eta^2 L^2\mu^2\left(\sum_{\tau=0}^{t-1}\mu^{2\tau}\right)
\expect\left[\sum_{\tau=0}^{t-1}\alpha_{t-\tau}^2\right]\nonumber\\
\le & \frac{\mu^{2t}\kappa^2}{b}+\frac{L\kappa\epsilon\mu^{t+1}}{\sqrt{b}}\frac{1-\mu^t}{1-\mu}+\frac{\eta L\kappa\mu^{t+1}}{\sqrt{b}}
\left(\sqrt{\frac{1-\mu^{2t}}{1-\mu^2}}\right)
\sqrt{\expect\left[\sum_{\tau=0}^{t-1}\alpha_{t-\tau}^2\right]}+\frac{\kappa^2}{b}\frac{(1-\mu)(1-\mu^{2t})}{1+\mu}\nonumber\\
& + L^2\mu^2\epsilon^2\left(\frac{1-\mu^{2t}}{1-\mu^2}\right) +\eta^2 L^2\mu^2\left(\frac{1-\mu^{2t}}{1-\mu^2}\right)
\expect\left[\sum_{\tau=0}^{t-1}\alpha_{t-\tau}^2\right].
\end{align}
Applying \cite[Lemma 3.2]{ward2020adagrad} gives
\[
\sum_{t=1}^T\sum_{\tau=0}^{t-1}\alpha_{t-\tau}^2= 
\sum_{t=1}^T\sum_{\tau=1}^{t}\alpha_\tau^2
= \sum_{t=1}^T (T-t+1) \alpha_t^2
\le T\sum_{t=1}^T  \alpha_t^2
\le T\ln\left(\left(\frac{\gamma^2}{v_0^2}T\right)+1\right).
\]
Also note that \[
\sum_{t=1}^T \frac{(1-\mu)(1-\mu^{2t})}{1+\mu} = \frac{1-\mu}{1+\mu} T-\frac{\mu^2}{(1+\mu)^2}(1-\mu^{2T})\le \frac{1-\mu}{1+\mu} T.
\]
Therefore, it follows from Eq. \eqref{eq:errorM} that
\begin{align}\label{eq:errorBM}
\sum_{t=0}^T \expect\left[\normF{E_{t+1}}^2\right]\le & \frac{1}{1-\mu^2}\frac{\kappa^2}{b}
+\frac{ L\kappa\epsilon}{(1-\mu)^2\sqrt{b}}+ \frac{1-\mu}{1+\mu}\frac{\kappa^2 T}{b}
+ \frac{\eta^2L^2\mu^2 }{1-\mu^2}T\left(\ln\left(\left(\frac{\gamma^2}{v_0^2}T\right)+1\right)\right)\nonumber\\
& + \frac{ L^2\mu^2\epsilon^2 T}{1-\mu^2}+\frac{\eta L\kappa}{\sqrt{b}\sqrt{1-\mu^2}}\sqrt{\left(\sum_{t=1}^T\mu^{2t+2}\right)
\left(\sum_{t=1}^T\sum_{\tau=0}^{t-1}\alpha_{t-\tau}^2\right)}\nonumber\\
\le & \frac{1}{1-\mu^2}\frac{\kappa^2}{b}+\frac{ L\kappa\epsilon}{(1-\mu)^2\sqrt{b}}+ \frac{1-\mu}{1+\mu}\frac{\kappa^2 T}{b}
+ \frac{\eta^2L^2\mu^2 }{1-\mu^2}T\left(\ln\left(\left(\frac{\gamma^2}{v_0^2}T\right)+1\right)\right)\nonumber\\
&+ \frac{ L^2\mu^2\epsilon^2 T}{1-\mu^2}+\frac{\mu\eta L\kappa}{(1-\mu^2)\sqrt{b}}\sqrt{ T\ln\left(\left(\frac{\gamma^2}{v_0^2}T\right)+1\right)}.
\end{align}
Then by Eq. \eqref{eq:totalM} and Eq. \eqref{eq:csM},
\begin{align*}
&\sum_{t=1}^T \expect\left[ \max\left\{\epsilon,\eta\alpha_t\right\}\norm{\Grad\L(\Theta_{t-1})}_*\right]\nonumber\\
\le &  \L(\Theta_0)-\L_\star+2\eta\sqrt{\sum_{t=1}^T \expect\left[\norm{E_{t}}_*^2\right]}
\sqrt{\ln\left(\frac{\gamma^2}{v_0^2}T\right)+1}+2\epsilon\sum_{t=1}^T\expect\left[\norm{E_t}_*\right]
+ \frac{ L}{2}\epsilon^2 T+ \frac{\eta^2 L}{2}\sum_{t=1}^T\expect\left[\alpha_t^2\right]\\
\le &\L(\Theta_0)-\L_\star+2\eta\sqrt{r\sum_{t=1}^T \expect\left[\normF{E_{t}}^2\right]}
\sqrt{\ln\left(\frac{\gamma^2}{v_0^2}T\right)+1}
+2\epsilon\sum_{t=1}^T\expect\left[\norm{E_t}_*\right] +\frac{L\epsilon^2 T}{2}
+ \frac{\eta^2 L}{2}\sum_{t=1}^T\expect\left[\alpha_t^2\right],
\end{align*}
where $r:=\min\{m,n\}.$
Combining the above with Eq. \eqref{eq:errorN} and Eq. \eqref{eq:errorBM} gives
\begin{align}\label{eq:long}
&\frac{1}{T}\sum_{t=1}^T
\expect\left[\norm{\Grad\L(\Theta_{t-1})}_*\right]\nonumber\\
\le & \frac{1}{T\epsilon}\sum_{t=1}^T\expect\left[ \max\left\{\epsilon,\eta\alpha_t\right\}\norm{\Grad\L(\Theta_{t-1})}_*\right]\nonumber\\
\le& \frac{\Delta}{\epsilon T}+\frac{ L\epsilon}{2}
+\frac{\eta^2 L}{2\epsilon T}\left(\ln\left(\frac{\gamma^2}{v_0^2}T\right)+1\right)
+\frac{2}{T}\sum_{t=1}^T\expect\left[\norm{E_t}_*\right]+ \frac{2\eta}{T\epsilon}\sqrt{r\sum_{t=1}^T \expect\left[\normF{E_{t}}^2\right]}
\sqrt{\ln\left(\frac{\gamma^2}{v_0^2}T\right)+1}\nonumber\\
\le &\frac{\Delta}{\epsilon T}+ \frac{L\epsilon}{2}
+\frac{\eta^2 L}{2\epsilon T}\left(\ln\left(\frac{\gamma^2}{v_0^2}T\right)+1\right)
+2\left(\frac{\kappa\sqrt{r}}{(1-\mu)T\sqrt{b}}+\kappa\sqrt{\frac{r(1-\mu)}{b}}+\frac{\mu L\epsilon}{1-\mu}\right.\nonumber\\
& \left. +\frac{\mu\eta L}{(1-\mu)\sqrt{T}}\left(\ln\left(\frac{\gamma^2}{v_0^2}T\right)+1\right)\right)
+ \frac{2\eta\sqrt{r}}{\sqrt{T}\epsilon}\left(\sqrt{\ln\left(\frac{\gamma^2}{v_0^2}T\right)+1}\right)
\left(\frac{\sqrt{ L\kappa\epsilon}}{(1-\mu)\sqrt{T}b^{1/4}}+\frac{L\mu\epsilon}{\sqrt{1-\mu^2}}\right.\nonumber\\
&
\left.+\frac{\kappa}{\sqrt{bT}\sqrt{1-\mu^2}}+\sqrt{\frac{1-\mu}{1+\mu}}\frac{\kappa}{\sqrt{b}}
+ \frac{\eta L\mu}{\sqrt{1-\mu^2}}\sqrt{\ln\left(\frac{\gamma^2}{v_0^2}T\right)+1}
+\frac{1}{(bT)^{1/4}}\sqrt{\frac{\mu\eta L\kappa}{1-\mu^2}}\left(\ln\left(\frac{\gamma^2}{v_0^2}T\right)+1\right)^{1/4}
\right).
\end{align}

In particular, if choosing $b=1$, $\epsilon=T^{-\frac{3}{4}}$,
$1-\mu=T^{-\frac{1}{2}}$, and $\eta= T^{-\left(\frac{3}{8}+q\right)}$ for arbitrary $q>0$, then by Eq. \eqref{eq:long}, \begin{align*}
&\frac{1}{T}\sum_{t=1}^T \expect\left[\norm{\Grad\L(\Theta_{t-1})}_*\right]\\
\le &\frac{\Delta}{T^{1/4}}+\frac{L}{2T^{3/4}}
+\frac{L}{2T^{1+2q}}\left(\ln T+1\right)
+2\left(\frac{\kappa\sqrt{r}}{T^{1/2}}+\frac{\kappa \sqrt{r}}{T^{1/4}}+\frac{L}{T^{1/4}}+\frac{L}{T^{3/8+q}}\left(\ln \left(\frac{\gamma^2}{v_0^2}T\right)+1\right)\right)\nonumber\\
& + 2\left(\sqrt{\ln T+1}\right)
\left(\frac{\sqrt{ L\kappa r}}{T^{1/2+q}}+\frac{L\sqrt{r}}{T^{5/8+q}}+\frac{2\kappa\sqrt{r}}{T^{3/8+q}}
+\frac{L\sqrt{r}}{T^{1/4+q}}\sqrt{\ln T+1}+\frac{\sqrt{rL\kappa}}{T^{5/16+3q/2}}\left(\ln \left(\frac{\gamma^2}{v_0^2}T\right)+1\right)^{\frac{1}{4}}\right)\nonumber\\
= & \bigO\left(\frac{\Delta+\kappa\sqrt{r}+L}{T^{1/4}}+\frac{L\sqrt{r}}{T^{1/4+q}}\left(\ln \left(\frac{\gamma^2}{v_0^2}T\right)+1\right)\right),
\end{align*}
for large $T$. The proof is thus completed.
\end{proof}

\section{Proof of Theorem~\ref{thm:determin}}\label{appendix:determin}
\begin{proof}
Write \[
\alpha_t:= \frac{\min\{\norm{\Grad\L(\Theta_t)},\gamma\}}{v_t}.
\]
By Lemma~\ref{lemma:lips},
\begin{align*}
\L(\Theta_{t})-\L(\Theta_{t-1})
\le & -\dotP{\Grad\L(\Theta_{t-1})}{\max\{\epsilon,\eta\alpha_t\}O_t} +\frac{ L}{2}\max\{\epsilon^2,\eta^2\alpha_t^2\}\\
\le & -\max\{\epsilon,\eta\alpha_t\}\norm{\Grad\L(\Theta_{t-1})}_*+\frac{L\epsilon^2}{2}+\frac{\eta^2 L}{2}\alpha_t^2.
\end{align*}
Hence, by Lemma~\ref{lem:as}, \[
\sum_{t=1}^T \max\{\epsilon,\alpha_t\}\norm{\Grad\L(\Theta_{t-1})}_*\le\Delta+\frac{L\epsilon^2 T}{2}+\frac{\eta^2 L}{2}\left(
\ln\left(\frac{\gamma^2}{v_0^2}T\right)+1\right).
\]
It follows that \begin{align*}
\frac{1}{T}\sum_{t=1}^T \norm{\Grad\L(\Theta_{t-1})}_*\le& \frac{1}{T\epsilon}\sum_{t=1}^T \max\{\epsilon,\alpha_t\}\norm{\Grad\L(\Theta_{t-1})}_*\nonumber\\
\le & \frac{\Delta}{T\epsilon}+\frac{L\epsilon}{2}+\frac{\eta^2 L}{2T\epsilon}
\left(\ln\left(\frac{\gamma^2}{v_0^2}T\right)+1\right).
\end{align*}

In particular, if choosing $\epsilon = T^{-\frac{1}{2}}$ and $\eta=T^{-q}$ for arbitrary $q>0,$ then
\[
\frac{1}{T}\sum_{t=1}^T \expect\left[\norm{\Grad\L(\Theta_{t-1})}_*\right]\le \frac{2\Delta+L}{2\sqrt{T}}+\frac{L}{2T^{2q+1/2}}
\left(\ln \left(\frac{\gamma^2}{v_0^2}T\right)+1\right).
\]
For large $T>0$, 
\[
\frac{1}{T}\sum_{t=1}^T \expect\left[\norm{\Grad\L(\Theta_{t-1})}_*\right]\le \bigO\left(\frac{\Delta+L}{\sqrt{T}}\right)
\]
The proof is thus completed.
\end{proof}

\section{Proof of Theorem~\ref{thm:sg}}\label{appendix:sg}
\begin{proof}
Write \[
\alpha_t:= \frac{\min\{\norm{G_t},\gamma\}}{v_t}.
\]
Let $\expect_t[\cdot]:=\expect[\cdot|\Theta_{t-1}]$ denote the conditional expectation given the previous iterates
$\Theta_0,\cdots,\Theta_{t-1}.$ By Lemma~\ref{lemma:lips},
\begin{align*}
&\expect_t\left[\L(\Theta_{t})-\L(\Theta_{t-1})\right]\\
\le & \expect_t\left[-\dotP{\Grad\L(\Theta_{t-1})}
{\max\{\epsilon,\eta\alpha_t\}O_t}\right] + \frac{L}{2}\expect_t\left[\max\{\epsilon^2,\eta^2\alpha_t^2\}\right]\\
= & \expect_t\left[-\dotP{\Grad\L(\Theta_{t-1})-G_t}{\max\{\epsilon,\eta\alpha_t\}O_t}\right] 
-\expect_t\left[\max\{\epsilon,\eta\alpha_t\}\norm{G_t}_*\right] + \frac{L\epsilon^2}{2}+
\frac{\eta^2 L}{2}\expect_t\left[\alpha_t^2\right]\\
\le & \expect_t\left[\norm{\Grad\L(\Theta_{t-1})-G_t}_*\left(\epsilon+\eta\alpha_t\right)\right]
-\expect_t\left[\max\{\epsilon,\eta\alpha_t\}\norm{G_t}_*\right]
+ \frac{L\epsilon^2}{2}+
\frac{\eta^2 L}{2}\expect_t\left[\alpha_t^2\right].
\end{align*}
Write $E_t:=\Grad\L(\Theta_{t-1})-G_t$. It follows that \begin{align*}
\expect_t\left[\norm{G_t}_*\right] \le & \frac{1}{\epsilon}\expect_t\left[\max\{\epsilon,\eta\alpha_t\}\norm{G_t}_*\right]\\
\le & \expect_t\left[\frac{\L(\Theta_{t-1})-\L(\Theta_{t})}{\epsilon}\right]+
\expect_t\left[\norm{E_t}_*\right]+\frac{\eta}{\epsilon}\expect_t\left[\norm{E_t}_*\alpha_t\right]
+ \frac{L\epsilon}{2}+
\frac{\eta^2 L}{2\epsilon}\expect_t\left[\alpha_t^2\right].
\end{align*}
Then by Jensen's inequality,
\begin{align}\label{eq:sgs}
\frac{1}{T}\sum_{t=1}^T \expect\left[\norm{\Grad\L(\Theta_{t-1})}_*\right]\le & \frac{1}{T}\sum_{t=1}^T \expect_t\left[\norm{G_t}_*\right]\nonumber\\
\le & \frac{\Delta}{\epsilon T}+\frac{1}{T}\sum_{t=1}^T\sqrt{\expect_t\left[\norm{E_t}_*^2\right]}
+ \frac{\epsilon L}{2}+
\frac{\eta^2 L}{2\epsilon T}\sum_{t=1}^T\expect\left[\alpha_t^2\right]+\frac{\eta}{T\epsilon}\sum_{t=1}^T\expect\left[\norm{E_t}_*\alpha_t\right]\nonumber\\
\le &  \frac{\Delta}{\epsilon T}+ \frac{\epsilon L}{2}+\frac{\kappa\sqrt{r}}{T}\sum_{t=1}^T\frac{1}{\sqrt{b_t}}
+\frac{\eta^2 L}{2\epsilon T}\sum_{t=1}^T\expect\left[\alpha_t^2\right]+\frac{\eta}{T\epsilon}\sum_{t=1}^T\expect\left[\norm{E_t}_*\alpha_t\right].
\end{align}
By Lemma~\ref{lem:as}, \begin{equation}\label{eq:ln}
\sum_{t=1}^T\expect\left[\alpha_t^2\right]\le\ln\left(\frac{\gamma^2}{v_0^2}T\right)+1.
\end{equation}
Then by Cauchy-Schwarz inequality, \begin{align}\label{eq:csln}
\sum_{t=1}^T\expect\left[\norm{E_t}_*\alpha_t\right]\le &\sum_{t=1}^T\sqrt{\expect\left[\norm{E_t}_*^2\right]\expect\left[\alpha_t^2\right]}\nonumber\\
\le & \sqrt{\sum_{t=1}^T\frac{\kappa^2 r}{b_t}}\sqrt{\sum_{t=1}^T\expect\left[\alpha_t^2\right]}\nonumber\\
\le & \kappa \sqrt{\sum_{t=1}^T\frac{r}{b_t}}\sqrt{\ln\left(\frac{\gamma^2}{v_0^2}T\right)+1}.
\end{align}
Combining Eq. \eqref{eq:sgs}, Eq. \eqref{eq:ln} and Eq. \eqref{eq:csln} gives
\begin{align*}
&\frac{1}{T}\sum_{t=1}^T\expect\left[\norm{\Grad\L(\Theta_{t-1})}_*\right]\nonumber\\
\le &\frac{\Delta}{\epsilon T}
+\frac{\epsilon L}{2}+\frac{\eta^2 L}{2T\epsilon}\left(\ln\left(\frac{\gamma^2}{v_0^2}T\right)+1\right)
+\frac{\kappa\sqrt{r}}{T}\left(\sum_{t=1}^T \frac{1}{\sqrt{b_t}}\right)+\frac{\kappa\eta}{T\epsilon}\sqrt{\sum_{t=1}^T \frac{r}{b_t}}\sqrt{\ln\left(\frac{\gamma^2}{v_0^2}T\right)+1}.
\end{align*}

In particular, if the batch size $b_t=\sqrt{t}$, $\epsilon = T^{-\frac{1}{2}}$ and $\eta=T^{-q}$ for 
arbitrary $q>0,$ then
\begin{align*}
\frac{1}{T}\sum_{t=1}^T\expect\left[\norm{\Grad\L(\Theta_{t-1})}_*\right]
\le & \frac{\Delta}{\sqrt{T}}+\frac{L}{2\sqrt{T}}+\frac{4\kappa\sqrt{r}}{3T^{\frac{1}{4}}}+\frac{\sqrt{2r}\kappa}{T^{\frac{1}{4}+q}}
\sqrt{\ln\left(\frac{\gamma^2}{v_0^2}T\right)+1}+\frac{L}{2T^{\frac{1}{2}+2q}}\left(\ln\left(\frac{\gamma^2}{v_0^2}T\right)+1\right)\\\nonumber
= & \bigO\left(\frac{\kappa\sqrt{r}}{T^{\frac{1}{4}}}+\frac{L+\Delta}{\sqrt{T}}\right)
\end{align*}
for large $T$.
Alternatively, if the batch size $b_t=t$, $\epsilon = T^{-\frac{1}{2}}$ and $\eta=T^{-q}$ for 
arbitrary $q>0,$ then
\begin{align*}
\frac{1}{T}\sum_{t=1}^T\expect\left[\norm{\Grad\L(\Theta_{t-1})}_*\right]
\le & \frac{\Delta}{\sqrt{T}}+\frac{L}{2\sqrt{T}}+\frac{2\kappa\sqrt{r}}{\sqrt{T}}+\frac{\kappa\sqrt{r}}{T^{\frac{1}{2}+2q}}
\left(\ln\left(\frac{\gamma^2}{v_0^2}T\right)+1\right)+\frac{L}{2T^{\frac{1}{2}+2q}}\left(\ln\left(\frac{\gamma^2}{v_0^2}T\right)+1\right)\\\nonumber
= & \bigO\left(\frac{\kappa\sqrt{r}+L+\Delta}{\sqrt{T}}\right)
\end{align*}
The proof is thus completed.
\end{proof}

\end{document}